\documentclass[letterpaper,10pt,conference]{ieeeconf}
\usepackage{amsmath,amsfonts,amssymb}
\usepackage{algorithmic}
\usepackage{algorithm}
\usepackage{array}
\usepackage{tabularx}
\usepackage{booktabs}
\usepackage{threeparttable}
\newcolumntype{Y}{>{\centering\arraybackslash}X}
\newcolumntype{L}{>{\raggedright\arraybackslash}X}
\usepackage{textcomp}
\usepackage{stfloats}
\usepackage{url}
\usepackage{verbatim}
\usepackage{graphicx}
\usepackage{subcaption}
\usepackage{cite}

\usepackage{bm}
\usepackage{pifont}
\newtheorem{lemma}{Lemma}

\usepackage{xcolor}

\graphicspath{{figs/}}
\newcommand{\mbdtrue}{\textsc{MBD-true}}
\newcommand{\dkmbd}{\textsc{DK-MBD}}
\newcommand{\dksplit}{\textsc{DK-MBD-split}}
\newcommand{\bkmbd}{\textsc{BK-MBD}}
\newcommand{\dklarge}{\textsc{DK-MBD-large}}
\newcommand{\mlpmbd}{\textsc{MLP-MBD}}

\usepackage{hyperref}
\usepackage{subcaption}
\DeclareCaptionFont{timesnr}{\fontfamily{ptm}\selectfont}
\IEEEoverridecommandlockouts                              

\title{\LARGE Koopman-Accelerated Model-Based Diffusion for Real-Time Robot Control}

\author{Bohyeong Pak$^{1}$, Kangmin Lee$^{1}$, and Sanghyun Kim$^{1,*}$
\thanks{This work was supported by the Korea Planning \& Evaluation Institute of Industrial Technology (KEIT) and the Ministry of Trade, Industry \& Energy (MOTIE) (RS-2025-11082970, RS-2025-25449039), and by the Korea Basic Science Institute (KBSI) through a grant funded by the Ministry of Science and ICT (MSIT) (RS-2025-00564593).}
\thanks{$^{1}$Department of Mechanical Engineering, Kyung Hee University, Yongin, Republic of Korea.}%
\thanks{$^{*}$Corresponding author: Sanghyun Kim. {\tt\small kim87@khu.ac.kr}}
}

\begin{document}

\maketitle
\thispagestyle{empty}
\pagestyle{empty}

\begin{abstract}
Conventional model-based diffusion (MBD) achieves effective trajectory optimization by leveraging noise annealing. However, its high computational cost, primarily arising from repeated rollouts of the plant dynamics, has largely confined its use to offline settings. To address this limitation, this paper proposes bilinear Koopman model-based diffusion (\bkmbd{}). The proposed method lifts the robot's state into a high-dimensional space only once per control step and propagates all candidates in the lifted space thereafter, so each rollout reduces to a fixed number of matrix--vector multiplications. The lifted dynamics are bilinear, allowing the predicted input gain to vary with the robot's configuration, which a linear lifted model cannot represent. In simulation, \bkmbd{} completed each planning update in at most $14.7$ ms within a $50$ ms control period and reached the goal on every trial, whereas a linear lift almost never did. The annealed schedule improves closed-loop accuracy over fixed-noise schedules under the learned rollout. Under the exact rollout, both the annealed and fixed-narrow schedules reach every goal, indicating that annealing reduces sensitivity to surrogate-model error. \bkmbd{} also threaded a passage that no single convex region covers, whereas a convexified bilinear controller rarely succeeded. On a physical manipulator, \bkmbd{} tracked an initially unknown moving target within the control period and was the only method that met both the tracking task and the deadline. The project page is available at \url{https://rcilab.khu.ac.kr/bkmbd/}.
\end{abstract}


\section{Introduction}\label{sec:intro}

Sampling-based predictive control is widely used for robotic systems with nonlinear dynamics.
Model predictive path integral (MPPI)~\cite{williams2016aggressive, williams2017mppi} and model-based diffusion (MBD)~\cite{pan2024model} evaluate candidate control sequences via forward rollout and update the nominal sequence using a cost-weighted average.
MBD treats each such update as a step of reverse diffusion.
The weighted average estimates the score of a Gaussian-smoothed target~\cite{song2019generative, song2020score}, while a decreasing noise schedule guides the search from coarse exploration to fine refinement.
This approach requires neither a differentiable nor a convex cost function and can accommodate nonconvex feasible sets.

However, MBD remains limited to offline trajectory optimization, leaving its application to receding-horizon control as an open challenge~\cite{pan2024model}.
The primary bottleneck is the computational cost of the rollout.
A single control step evaluates candidate rollouts at every annealing stage, and receding-horizon control draws a new population after each executed action, so a simulation-based rollout is evaluated far too often to be affordable.
Xue et al.~\cite{xue2025full} mitigated this cost with a vectorized physics engine on a GPU.
On a CPU, meeting the control rate still calls for a rollout model far lighter than a full simulation.

A Koopman lift enables such low-cost computation by acting as a surrogate for the plant.
Lifting the state into a high-dimensional space through hand-designed~\cite{williams2015edmd, proctor2016dmdc} or learned~\cite{lusch2018deep} observables turns the nonlinear rollout into simple matrix--vector multiplications.
Hao et al.~\cite{mppidk2026} applied this by integrating a linear Koopman model into MPPI.
Placing a finite-dimensional learned model inside a sampling planner, however, raises two challenges.
The rollout must preserve the configuration-dependent effect of the input, and the optimizer must be robust to local minima that residual model error may introduce into the surrogate cost landscape.

The first challenge originates in how the input is represented.
A linear lift processes the input through a constant matrix, which amounts to assuming an input gain that does not depend on the configuration of the robot.
On a velocity-commanded robot, the input gain instead varies continuously with the configuration.
This mismatch is particularly detrimental to sampling-based planning, where candidates are selected according to their predicted cost differences.
Holding the gain constant misrepresents how different control sequences affect the robot across configurations and can therefore corrupt the ranking of candidates. Increasing the number of candidates does not remove this error, since it is structural rather than statistical.
A bilinear Koopman model allows the predicted input gain to depend on the configuration through state--input product terms~\cite{iacob2024}. Linear lifted models have nonetheless remained attractive because they allow the control problem to be cast as a quadratic program (QP)~\cite{korda2018mpc}, and Bruder et al.~\cite{bruder2021bilinear} retained bilinear dynamics but froze the state-dependent terms along a nominal trajectory to preserve convexity, yielding a local approximation of the full state--input coupling. Thus, the practical obstacle lies not in the bilinear representation itself, but in optimizing it using conventional convex methods.

The second challenge originates in model error.
A learned rollout model is approximate, and its residual error may introduce local minima into the surrogate cost landscape.
An optimizer that settles into one of these minima selects a sequence that looks favorable under the surrogate yet performs poorly on the real system.
Raising the expressiveness of the rollout model is thus not sufficient on its own, and the optimizer has to be robust to such error as well.

This paper accordingly proposes bilinear Koopman model-based diffusion (\bkmbd{}), which combines a bilinear Koopman rollout with the annealed sampling of MBD.
The bilinear structure carries the configuration-dependent input gain, and the sampling planner evaluates only the total cost of a sequence without requiring convexity, so it uses that state dependence as it stands rather than freezing or convexifying it.
The annealed schedule smooths the cost-induced target broadly at high noise levels and refines the plan under a less-smoothed target as the noise decreases, which can reduce sensitivity to these surrogate-induced local minima. To test whether this benefit originates in surrogate error, fixed-noise and annealed schedules are compared under both learned and exact rollouts.
Computationally, the state is lifted once per control step and every candidate is then propagated in the lifted space, so each rollout reduces to a fixed number of matrix--vector multiplications.

The diffusion in \bkmbd{} is a property of the optimizer rather than a learned motion prior.
Existing diffusion-based motion planners learn task-conditioned trajectory distributions from task-specific data~\cite{janner2022planning, chi2025diffusion, carvalho2023motion}, whereas \bkmbd{} learns only the rollout model; goals and constraints enter at deployment through the cost function.

The main contributions are as follows.

\begin{itemize}
\item \textbf{MBD at the control rate.}
\bkmbd{} employs a bilinear Koopman rollout that represents the configuration-dependent input gain without freezing or convexification. Lifting the state once per control step reduces all subsequent candidate propagation to a fixed sequence of matrix--vector multiplications, allowing the full annealing process to run within the control period.

\item \textbf{Annealing under a learned rollout.}
Under the learned surrogate, annealing attains higher reaching accuracy than every fixed-noise schedule. Under the exact rollout, both the annealed and fixed-narrow schedules reach every goal, indicating that annealing reduces sensitivity to surrogate-model error.

\item \textbf{Real-time validation on a physical robot.}
On a physical FR3 tracking a moving target whose trajectory is unavailable in advance, \bkmbd{} completes all ten trials and returns every command within the 50~ms control period, whereas every baseline fails the task, the deadline, or both.
\end{itemize}

\section{Preliminaries}\label{sec:prelim}

In this section we formulate the control problem, review the MBD update, and state the two classes of lifted rollout that the study compares.

\subsection{Problem Formulation}\label{subsec:problem}

Consider the discrete-time system
\begin{equation}
s_{t+1} = f(s_t, u_t), \quad
s_t \in \mathcal{X} \subset \mathbb{R}^{n},\;
u_t \in \mathcal{U} = [\underline{u}, \bar{u}]^{m},
\label{eq:sys}
\end{equation}
obtained from a control period $\Delta t$.
Here $f$ is unknown or too costly to integrate at the sampling rate, so a true-dynamics rollout serves as an oracle reference alone, and the rollout model is learned from interaction data $\mathcal{D} = \{(s_i, u_i, s_i^{+})\}$.
Throughout, $\mathcal{X}$ is the operating set, the region to which the task confines the plant, and every statement below is made over it.
The task defines features $b(s) \in \mathbb{R}^{d}$, which the rollout model predicts and the cost reads.
The controller tracks an output $y = S_y\, b(s) \in \mathbb{R}^{d_y}$ selected from those features by a constant $S_y$ of full row rank.

The plants considered here are control-affine and velocity-commanded, $\dot{s} = f_c(s) + G(s)u$ with output $y = h(s)$.
Integrating over one control period under a zero-order hold gives the response of the tracked output,
\begin{equation}
y_{t+1} - y_t = \big(\omega(s_t) + \Phi(s_t)\, u_t\big)\Delta t + O(\Delta t^{2}),
\label{eq:resp}
\end{equation}
with drift response $\omega = \nabla h\, f_c \in \mathbb{R}^{d_y}$ and \emph{input gain} $\Phi = \nabla h\, G \in \mathbb{R}^{d_y \times m}$.
The drift response is a function of the state alone and a model of the state may absorb it, so $\Phi(s)u$ is the only term in which the state and the input meet.
The input gain is in general state-dependent.
How much it varies over $\mathcal{X}$ is a property of the plant and the task rather than of the controller applied to them.

Everything the controller optimizes is produced by a \emph{rollout model} $\mathcal{R}: (b_t, U) \mapsto (\hat{b}_1, \dots, \hat{b}_T)$ acting on a candidate sequence $U = (u_0, \dots, u_{T-1})$.
The cost depends on the candidate only through $\mathcal{R}$, so two controllers that share an optimizer differ exactly in $\mathcal{R}$.
At state $s_t$ the controller solves
\begin{equation}
\begin{aligned}
\min_{U} \;\; & J(U) = \sum_{k=1}^{T} c(\hat{b}_k, u_{k-1}) + \ell_T(\hat{b}_T) \\
\text{s.t.} \;\; & U \in \mathcal{U}^{T}, \quad \hat{b}_k \in \mathcal{F}, \quad k = 1, \dots, T ,
\end{aligned}
\label{eq:ocp}
\end{equation}
where $\mathcal{F}$ is the free set the task admits.
It then applies $u_0$ to \eqref{eq:sys} and replans at the next period.

Neither $c$ nor $\ell_T$ is assumed differentiable or convex in $U$, and $\mathcal{F}$ is assumed neither convex nor connected.
Since no structure is imposed on $\mathcal{F}$, an optimizer may enforce it as a constraint or absorb it into the cost as a penalty.
The controller is real-time in the sense that $u_0$ must be returned within $\Delta t$, so the admissibility of a rollout model depends on what the optimizer asks of it in one control step.

\subsection{The Model-Based Diffusion Update}\label{subsec:mbd}

MBD~\cite{pan2024model} treats \eqref{eq:ocp} as sampling from the Boltzmann target $p_0(U) \propto \exp(-J(U)/\alpha)$ with temperature $\alpha > 0$.
It descends a ladder of Gaussian-smoothed marginals $p_{\sigma} = p_0 * \varphi_{\sigma}$ over a decreasing schedule $\sigma_1 > \cdots > \sigma_S$.\footnote{We write the ladder in variance-exploding form, in which the smoothing kernel alone carries the schedule. The formulation of~\cite{pan2024model} additionally rescales the iterate, which affects no statement below.}
Applying Tweedie's formula to a smoothed marginal expresses its score as a denoising mean shift,
\begin{equation}
\nabla_U \log p_{\sigma}(U) = \big(\mathbb{E}[\,U_0 \mid U\,] - U\big)\big/\sigma^{2} .
\label{eq:tweedie}
\end{equation}
Since $\varphi_\sigma$ is the density of the proposal $U' \sim \mathcal{N}(U, \sigma^2 I)$, the posterior mean in \eqref{eq:tweedie} is a ratio of expectations under that proposal, and drawing $N$ candidates $U_k = U + \sigma \varepsilon_k$ with $\varepsilon_k \sim \mathcal{N}(0,I)$ and projecting them onto $\mathcal{U}^{T}$ gives its self-normalized estimator,
\begin{equation}
\widehat{\mathbb{E}}[U_0 \mid U] = \sum_{k=1}^{N} w_k U_k,
\quad
w_k = \frac{e^{-(J(U_k) - J_{\min})/\alpha}}{\sum_{l} e^{-(J(U_l) - J_{\min})/\alpha}} .
\label{eq:weights}
\end{equation}
The update at stage $j$ is a preconditioned score step with optional Langevin noise,
\begin{equation}
U \leftarrow U + \eta_j\, \widehat{\nabla_U \log p_{\sigma_j}}(U) + \sqrt{2 \eta_j}\, \xi,
\qquad \eta_j = \eta\, \sigma_j^{2},
\label{eq:update}
\end{equation}
which at $\eta = 1$ and without Langevin noise, the setting used throughout, is the cost-weighted mean $U \leftarrow \sum_k w_k U_k$.

Three properties of \eqref{eq:update} are used later, and none of them depends on the rollout model that supplies $J$.
First, the weights \eqref{eq:weights} are invariant to any term shared by the population.
Only differences of cost within a population enter the update, and the candidates of a population differ in $U$ alone.
Second, the $S$ stages are serial, since stage $j$ samples around the iterate produced by stage $j-1$.
The $N$ candidates within a stage and the $T$ steps of a rollout admit batching.
Third, a single stage at a fixed noise level, with $\eta = 1$ and no Langevin noise, is one cost-weighted average over a Gaussian population drawn around the warm-started iterate.
That average is the MPPI update~\cite{williams2017mppi}, so the annealed schedule is the property separating \eqref{eq:update} from the planner it generalizes, and a comparison at a matched total sample budget is available by construction.

The rollout model is evaluated once per candidate at every stage and over the full horizon, so a control step issues $N S T$ queries to $\mathcal{R}$.

\subsection{Linear and Bilinear Lifted Rollouts}\label{subsec:koopman}

The Koopman operator~\cite{koopman1931, mezic2005} advances observables of the state along the flow and is linear on the space of observables, but it is infinite-dimensional.
For a system driven by an input, an exact realization requires observables of the input as well as of the state~\cite{iacob2024}.
A finite-dimensional realization therefore truncates the representation of the state and, separately, fixes the form in which the input enters the lifted dynamics.
The two classes below agree in the first respect and differ in the second.

We use a learned lift with an exact-decoder convention, $z = \Psi_\theta(b) \in \mathbb{R}^{r}$, whose leading $d$ coordinates hold $b$ itself and whose remainder is learned.
Placing the features in the leading coordinates makes the decoder $C$ with $Cz = b$ a constant selector, and composing it with the output map gives the constant readout $C_y = S_y C \in \mathbb{R}^{d_y \times r}$.
Two classes of lifted dynamics are available,
\begin{align}
\text{linear:} \quad & z^{+} = A z + B u ,
\label{eq:lin}\\[2pt]
\text{bilinear:} \quad & z^{+} = A z + B_0 u + \textstyle\sum_{i=1}^{m} u_i B_i z ,
\label{eq:bil}
\end{align}
the latter written compactly as $z^{+} = M(u)z + B_0 u$ with $M(u) = A + \sum_i u_i B_i$.
For control-affine systems the continuous-time Koopman generator carries a bilinear input structure~\cite{iacob2024}, of which \eqref{eq:bil} is the finite-dimensional, discrete-time counterpart.
Realizations of this form have been identified from data and used for robot control~\cite{bruder2021bilinear}.

\begin{figure*}[t]
\centering
\includegraphics[width=\linewidth]{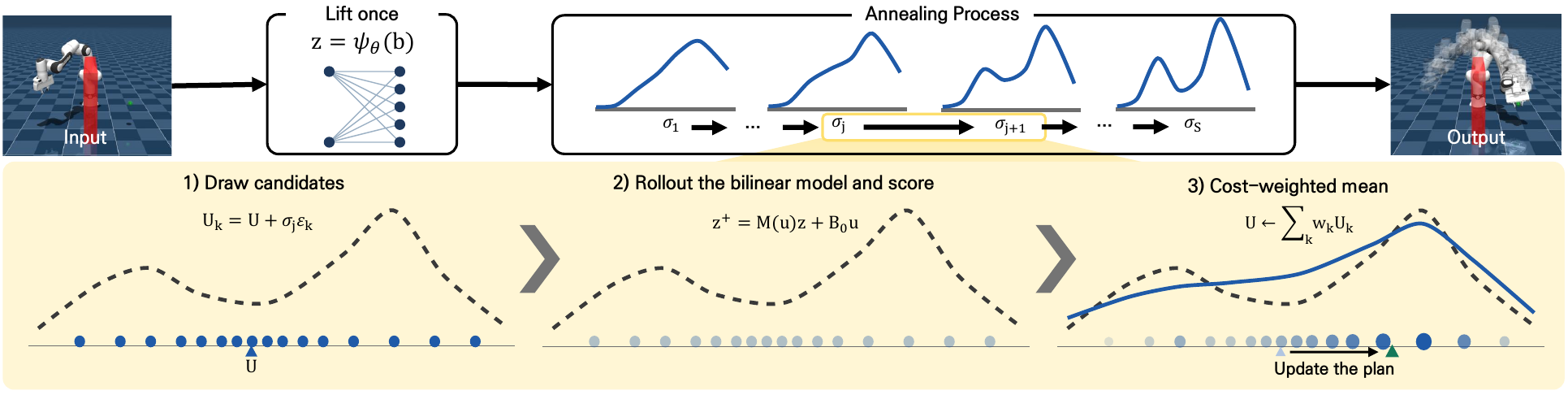}
\caption{One control step of \bkmbd{}. The state is lifted once, and every rollout of the annealed search is then a product of learned matrices.}
\label{fig:overview}
\end{figure*}

The two classes differ in exactly one place, the input tangent map.
For \eqref{eq:bil}, $\partial z^{+}/\partial u_i = B_0 e_i + B_i z$ depends on the state through $B_i z$, while for \eqref{eq:lin}, $\partial z^{+}/\partial u = B$ is one matrix for the whole of $\mathcal{X}$.
Both remain in the same computational regime, a bilinear step costing $m+1$ matrix--vector products of size $r$ against one for the linear form.
The cost of a rollout step is therefore set by the lifted dimension and the number of input channels alone.

\section{Bilinear Koopman Model-Based Diffusion}\label{sec:method}

In this section we propose \bkmbd{}, which keeps the annealed sampling optimizer of MBD and has every rollout performed by a lifted surrogate.
Sec.~\ref{subsec:bilinear} fixes the class of that surrogate and Sec.~\ref{subsec:training} identifies a model of that class from data.
The rollout is then a learned object rather than a simulator, and Sec.~\ref{subsec:anneal} treats the role this leaves to the noise schedule.
Fig.~\ref{fig:overview} shows one control step.

\subsection{The Bilinear Rollout}\label{subsec:bilinear}

Sec.~\ref{subsec:koopman} leaves the learned coordinates of the lift unspecified.
We take
\begin{equation}
z = \Psi_\theta(b) = \begin{bmatrix} b \\ \psi_\theta\big(\tau(b)\big) \end{bmatrix} \in \mathbb{R}^{r},
\label{eq:lift}
\end{equation}
where $\psi_\theta$ is a small neural network and $\tau$ is a fixed feature map that supplies terms the learned part would otherwise have to construct.
We write $\Psi_i$ for the $i$th component of $\Psi_\theta$.

What governs the optimizer is the response the rollout predicts to an applied input, since the weights \eqref{eq:weights} read only differences of cost within a population whose members differ in the control sequence alone.
Decoding one step of the linear class \eqref{eq:lin} gives $C_y z^{+} = C_y A z + C_y B u$, so that response has sensitivity $C_y B$, a single matrix over the whole of $\mathcal{X}$.
The lifted dimension $r$ and the encoder $\Psi_\theta$ change the value of $C_y B$ and leave it a single matrix, since $C_y$ is a constant selector and $B$ carries no dependence on the state.
The plant answers with $\Delta t\, \Phi(s)$ by \eqref{eq:resp}, which takes a different value at every configuration.

\begin{lemma}[Input-channel error of a linear lift]
\label{lem:linear}
Let $\nu_\Phi = \sup_{s, s' \in \mathcal{X}} \| \Phi(s) - \Phi(s') \|$ denote the variation of the input gain over the operating set.
For every constant matrix $M \in \mathbb{R}^{d_y \times m}$,
\begin{equation}
\sup_{s \in \mathcal{X}} \big\| \Delta t\, \Phi(s) - M \big\|
\;\ge\; \tfrac{1}{2}\, \nu_\Phi\, \Delta t .
\label{eq:bound}
\end{equation}
In particular the bound holds at $M = C_y B$ for every pair $(A, B)$ of \eqref{eq:lin}, and therefore at every lifted dimension and for every encoder.
\end{lemma}

\begin{proof}
For any $s, s' \in \mathcal{X}$ the triangle inequality gives
\begin{equation*}
\big\| \Delta t\, \Phi(s) - M \big\| + \big\| \Delta t\, \Phi(s') - M \big\|
\;\ge\; \Delta t \big\| \Phi(s) - \Phi(s') \big\| ,
\end{equation*}
so the larger of the two terms is at least $\tfrac{\Delta t}{2}\|\Phi(s) - \Phi(s')\|$.
Taking the supremum over pairs $(s, s')$ yields \eqref{eq:bound}.
The decoded input gain of \eqref{eq:lin} is one such constant matrix at every lifted dimension and for every encoder, which gives the final claim.
\end{proof}

The bound depends on the variation of the input gain alone, and the deficit is not confined to the first step, since $\partial (C_y z_k)/\partial u_j = C_y A^{\,k-j-1} B$ depends on the elapsed steps $k-j$ and not on the state at which the input is applied.
A linear rollout therefore predicts the same response to a given input at every configuration a candidate visits.
The same deficit underlies the linear--bilinear criterion of~\cite{lee2026linearbilinear}.

\begin{algorithm}[t]
\caption{One control step of \bkmbd{}}
\label{alg:bkmbd}
\footnotesize
\begin{algorithmic}[1]
\REQUIRE state $s_t$, warm-started nominal $U$, schedule $\sigma_1 > \cdots > \sigma_S$, temperature $\alpha$, model $(A, B_0, \{B_i\}_{i=1}^{m}, \Psi_\theta, C)$
\STATE $z \leftarrow \Psi_\theta\big(\tau(b(s_t))\big)$ \hfill $\triangleright$ encoder once per control step
\FOR{$j = 1$ to $S$}
\FOR{$k = 1$ to $N$, batched}
\STATE $U_k \leftarrow \mathrm{proj}_{\mathcal{U}^{T}}\!\big(U + \sigma_j \varepsilon_k\big)$, \quad $\varepsilon_k \sim \mathcal{N}(0, I)$
\STATE $z^{k}_{0} \leftarrow z$
\FOR{$l = 1$ to $T$}
\STATE $z^{k}_{l} \leftarrow M\big(u^{k}_{l-1}\big) z^{k}_{l-1} + B_0 u^{k}_{l-1}$ \hfill $\triangleright$ $m+1$ mat--vec
\STATE $\hat{b}^{k}_{l} \leftarrow C z^{k}_{l}$
\ENDFOR
\STATE $J_k \leftarrow \sum_{l=1}^{T} c\big(\hat{b}^{k}_{l}, u^{k}_{l-1}\big) + \ell_T\big(\hat{b}^{k}_{T}\big)$
\ENDFOR
\STATE $w_k \leftarrow$ softmax weights \eqref{eq:weights} at temperature $\alpha$
\STATE $U \leftarrow \sum_{k=1}^{N} w_k U_k$ \hfill $\triangleright$ stage update \eqref{eq:update}
\ENDFOR
\STATE shift $U$ by one step to warm-start the next period
\RETURN $u_0$
\end{algorithmic}
\end{algorithm}

The bilinear class is not subject to this bound, since its decoded input gain is not a constant matrix. That gain
\begin{equation}
\frac{\partial (C_y z^{+})}{\partial u_i} = C_y b_{0,i} + C_y B_i z ,
\qquad i = 1, \dots, m,
\label{eq:bilsens}
\end{equation}
is affine in the lifted state, with $b_{0,i}$ the $i$th column of $B_0$.
The predicted response then matches \eqref{eq:resp} to first order at every configuration whenever
\begin{equation}
C_y b_{0,i} + C_y B_i\, \Psi_\theta\big(b(s)\big) = \Delta t\, \Phi_i(s),
\qquad s \in \mathcal{X},
\label{eq:match}
\end{equation}
where $\Phi_i$ is the $i$th column of $\Phi$.
Since $C_y = S_y C$ composes two selectors of full row rank, the bilinear class can represent \eqref{eq:match} when the entries of $\Phi$ lie in the span of $\{1, \Psi_1, \dots, \Psi_r\}$ over $\mathcal{X}$.
Other classes whose input gain may depend on the state can also satisfy \eqref{eq:match}.
What \eqref{eq:bil} adds is that the dependence enters in the affine form that \eqref{eq:resp} exhibits.

On a velocity-commanded platform the map from the command frame to the world frame turns as the platform reconfigures.
It ranges over its full extent across the postures a task visits, so $\nu_\Phi$ stays bounded away from zero and \eqref{eq:bound} bars the linear class.
Such a map is built from sines and cosines of the configuration angles and is multilinear in those terms.
Letting $\tau$ supply those trigonometric functions places the leading terms in the lift and leaves $\psi_\theta$ to carry the products among them, the form \eqref{eq:match} asks for.
On a plant whose input gain varies little the bound is small, and \cite{lee2026linearbilinear} turns the variation of $\Phi$ into a test on a given operating set.

The cost follows from the count of Sec.~\ref{subsec:mbd}.
One control step lifts the measured features once and propagates every candidate in the lifted space afterwards, so the encoder is evaluated once per control step rather than once per candidate, which places it outside both loops of Algorithm~\ref{alg:bkmbd}.
The $S$ stages run in series, so the control period is divided among $S$ batched rollouts of $N$ candidates over $T$ steps.

\subsection{Identification}\label{subsec:training}

The lifted features and the matrices of the lifted dynamics are fitted jointly to recorded trajectories.
Koopman models are ordinarily identified by a one-step fit~\cite{lusch2018deep}, which suits a model queried one step at a time.
A sampling planner instead propagates a candidate over the whole horizon and separates candidates by the cost that propagation accumulates.
A model that is accurate at one step and drifts over the horizon has its drift read as a property of the plant.
We therefore match the fitting length to the planning horizon and fit every model on length-$T$ snippets by
\begin{equation}
\mathcal{L} = \frac{1}{T} \sum_{k=1}^{T}
\Big[ \big\| C \hat{z}_k - b_k \big\|^{2}
+ \gamma \big\| \hat{z}_k - \Psi_\theta(b_k)^{\mathrm{sg}} \big\|^{2} \Big] ,
\label{eq:loss}
\end{equation}
starting the rollout at $\hat z_0 = \Psi_\theta(b_0)$ and proceeding by \eqref{eq:lin} or \eqref{eq:bil}.
The first term measures the prediction error of the decoded features, which is the quantity the cost reads.
The second term, weighted by $\gamma > 0$, holds the propagated latent state to the image of the encoder, which \eqref{eq:lin} and \eqref{eq:bil} leave it free to leave.
The stop-gradient on the latent target prevents the encoder from lowering the loss by pulling the target toward the prediction.

The bilinear input matrices are initialized at $B_i = 0$.
Training then begins exactly at the linear model \eqref{eq:lin}, and the coupling term departs from zero under the fitting objective alone.
The two classes share a starting point, an encoder, a dataset and an objective, and they differ in whether the coupling term is permitted to move.

\subsection{Annealing under a Learned Rollout}\label{subsec:anneal}

A learned rollout evaluates $\hat{J} = J + \Delta J$, where $\Delta J$ collects the effect of model error.
The weights \eqref{eq:weights} are invariant to the part of $\Delta J$ common to the population, and the remaining part reorders candidates.
A learned rollout therefore admits minima of its own, and under receding horizon the robot executes whatever the sampler settles on.
The schedule cannot remove this error from the rollout.
Its role is instead to reduce how strongly one control step depends on a single local view of the surrogate cost.

Two properties of the update \eqref{eq:update} bear on this, and both belong to the sampling rather than to the model.
The first is how widely one stage samples.
The proposals $U_k = U + \sigma \varepsilon_k$ are spread about the current plan at a scale set by $\sigma$, so $\sigma$ fixes how far from that plan the cost is evaluated at all.
The same scale bounds the update, which is a convex combination of those proposals, so the width a stage explores is set primarily by the noise level.
That scale is an upper limit rather than a description, and a stage whose population holds nothing better than the current plan returns a weighted mean close to that plan at any $\sigma$.

The second is what one stage does with that width.
The update moves the nominal by $\sigma \sum_k w_k \varepsilon_k$, which carries $\sigma$ as a factor, so the noise level sets the size of the step.
It also sets where the step points, since by \eqref{eq:tweedie} the stage ascends $\log p_{\sigma}$ and $p_\sigma$ is the target smoothed at that level.
A stage therefore descends the cost smoothed at its own noise level.

One noise level fixes both quantities, so a fixed-noise sampler must trade exploration against refinement.
A small $\sigma$ gives local proposals, a less-smoothed cost and a small final displacement, but it can remain tied to a surrogate-induced local minimum.
A large $\sigma$ evaluates a broader neighborhood, but it reads a more heavily smoothed cost and produces a coarser update.
This does not establish that a larger noise level is better; it only characterizes the scale at which the stage reads and moves.
The empirical question is whether combining scales is preferable to assigning the same sample budget to one scale.

The annealed schedule separates the two roles across successive stages.
It opens at a wide level, where the population is spread furthest from the current plan, and closes at a narrow level, where the smoothing and the final displacement are smallest.
Under receding horizon the executed command comes from the last stage.
The early stages therefore expose the optimizer to a broader neighborhood of the current plan, while the last stage refines the executed command under a less-smoothed surrogate cost.
In the single-basin reaching task below, this suggests that the schedule's advantage should be tied to learned-rollout error, a point tested by repeating the schedules with an exact rollout.

\section{Experiments}\label{sec:exp}

\begin{table}[t]
\centering
\begin{threeparttable}
\caption{Implementation settings.}
\label{tab:hyper}
\footnotesize
\begin{tabularx}{\linewidth}{l Y Y}
\toprule
Setting & Manipulator & Drone \\
\midrule
Control period $\Delta t$ & $50$\,ms & $50$\,ms \\
Budget per trial & $120$ steps & $120$ steps \\
Candidates $N$, stages $S$ & $800$, $5$ & $800$, $5$ \\
Horizon $T$ & $15$ & $25$ \\
Noise $\sigma$ & $1.2 \to 0.3$ & $1.6 \to 0.3$ \\
Temperature $\alpha$ & $0.4$ & $0.5$ \\
Lift dimension $r$, latent weight $\gamma$ & $20$, $0.1$ & $13$, $0.1$ \\
Fixed feature map $\tau$ & $[\sin q, \cos q]$ & identity \\
Encoder $\psi_\theta$, layers $\times$ width & $2 \times 96$, tanh & $2 \times 64$, tanh \\
Training snippets, $T$ steps each & $6000$ & $4000$ \\
\bottomrule
\end{tabularx}
\end{threeparttable}
\vspace{-0.3cm}
\end{table}

We evaluate \bkmbd{} on two velocity-commanded platforms in simulation and on a physical manipulator.

The experiments that follow ask four questions.
What must the rollout model represent for the planner to reach the goal, and does it meet the control period?
Where does the advantage of the annealed schedule come from?
Does the structure of the free space separate the two optimizers on a passage no single convex region covers?
And do both requirements bind on a physical robot?

\subsection{Setup}\label{subsec:setup}

\textbf{Platforms.}
Sec.~\ref{subsec:bilinear} leaves the input gain to be instantiated on each platform, and Table~\ref{tab:hyper} gives the fixed feature map $\tau$ of \eqref{eq:lift} that follows from it.
The Franka Research 3 (FR3) accepts joint velocities through a gravity-compensated servo, tracks the tool center point (TCP) between its fingertips, and is observed as $b = [q, p_{\mathrm{tcp}}]$.
Its input gain is the manipulator Jacobian.\footnote{The FR3 is used deliberately because its analytic input gain provides a reference for testing whether the learned rollout recovers the same configuration dependence from data, while also enabling the \dksplit{} ablation. The goal is not to replace the known Jacobian, but to validate a rollout structure applicable when the corresponding input map is unknown or not readily available.}
The drone accepts body-frame linear velocities and a yaw rate, tracks its world-frame position, and is observed as $b = [p, \sin\psi, \cos\psi]$.
Its input gain is the yaw rotation.
The manipulator therefore carries the experiments of Secs.~\ref{subsec:exp_class} and~\ref{subsec:exp_anneal} and the hardware experiment of Sec.~\ref{subsec:exp_hw}, and the drone carries the non-convex passage of Sec.~\ref{subsec:exp_window}.

\textbf{Task and protocol.}
Every comparison below holds the optimizer or the rollout model fixed and exchanges the other.
In the reaching experiments each trial starts from rest and drives the tracked output toward a goal over the fixed budget of Table~\ref{tab:hyper}, without early termination.
A trial reaches a tolerance at the first control step whose tracking error falls below it, and we report the $5$\,cm and $1$\,cm tolerances.
Where a condition fails to reach, we report the tracking error at the end of the budget.
Candidates are scored by the squared tracking error with a control penalty and a terminal penalty.
Each dataset is collected from randomized initial conditions over the operating set under randomized velocity commands that keep a coherent direction across a snippet, so the configuration travels far enough for the state dependence of the input gain to be visited.
The reaching experiments cross $5$ training seeds, each with its own dataset, with $10$ goals for $50$ trials per condition, and the trials are paired across conditions.
A condition that carries no learned model is run over the $10$ goals alone.
All timings are measured on an Intel Core Ultra 5 225F with ten cores and no GPU. Simulation uses ten PyTorch CPU threads, whereas physical-robot trials limit PyTorch inference to four threads.

\subsection{Rollout Class and Planning Latency}\label{subsec:exp_class}

\begin{table}[t]
\centering
\begin{threeparttable}
\caption{Reaching accuracy under an identical planner.}
\label{tab:class}
\footnotesize
\begin{tabularx}{\linewidth}{l Y Y}
\toprule
Rollout model & Reach $5$\,cm & Reach $1$\,cm \\
\midrule
\mbdtrue{} & $10/10$ & $10/10$ \\
\dkmbd{} & $10/50$ & $1/50$ \\
\dklarge{} & $11/50$ & $1/50$ \\
\mlpmbd{} & $41/50$ & $0/50$ \\
\dksplit{} & $50/50$ & $50/50$ \\
\bkmbd{} (ours) & $\mathbf{50/50}$ & $\mathbf{50/50}$ \\
\bottomrule
\end{tabularx}
\end{threeparttable}
\end{table}

We settle whether the input gain is what decides reaching by exchanging the rollout model alone.
Table~\ref{tab:class} lists them by how the input gain enters.
\dkmbd{} leaves it a constant matrix.
\dklarge{} keeps that matrix but raises the lifted dimension, and \mlpmbd{} lets it depend on the state without structure, so the two separate capacity and expressiveness from the input gain.
\bkmbd{} represents the state-dependent gain inside the model, and \dksplit{} is the ablation that supplies it from the analytic forward kinematics instead.
\mbdtrue{} rolls out the true dynamics as the oracle reference. 

On the manipulator, \bkmbd{} reached the $1$\,cm tolerance on $50/50$ trials, matching the oracle, while \dkmbd{} reached it on $1/50$ and came to rest a median $0.248$\,m from the goal.
\dksplit{} also reached $50/50$, and it differs from \dkmbd{} in the input gain alone, so that is what separated the two, as Lemma~\ref{lem:linear} predicts.

\begin{figure}[t]
\centering
\captionsetup{skip=2pt}
\includegraphics[width=0.8\linewidth]{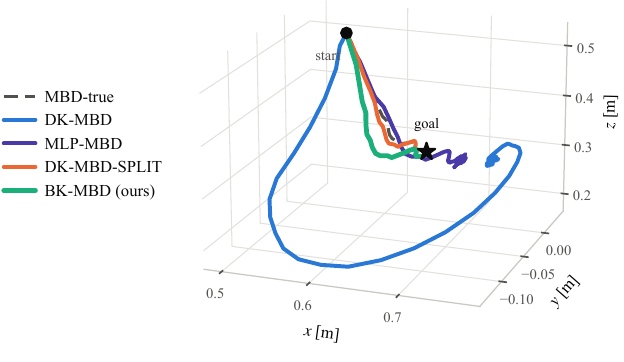}
\caption{Executed tool paths on one goal, one of the trials Table~\ref{tab:class} counts. The five runs share the planner, the training seed and the random stream.}
\label{fig:paths}
\end{figure}

\dklarge{} and \mlpmbd{} stopped a median $0.239$\,m and $0.054$\,m away.
The bound \eqref{eq:bound} carries no lifted dimension, so a larger dictionary still leaves a single matrix, and \mlpmbd{} can represent a state-dependent gain but did not identify it from the same data under the same objective.
Neither capacity nor expressiveness therefore removes the failure, which leaves the input gain as the cause. Fig.~\ref{fig:paths} draws one goal of the comparison.

\begin{table}[t]
\centering
\begin{threeparttable}
\caption{End-to-end planning latency per receding-horizon control step over all trials.}
\label{tab:latency}
\footnotesize
\begin{tabularx}{\linewidth}{l Y Y Y}
\toprule
Rollout model & Median [ms] & Worst [ms] & Misses \\
\midrule
\mbdtrue{} & $2598$ & $2745$ & $1200$ \\
\dkmbd{} & $7.6$ & $12.5$ & $0$ \\
\bkmbd{} (ours) & $11.0$ & $14.7$ & $\mathbf{0}$ \\
\bottomrule
\end{tabularx}
\end{threeparttable}
\end{table}

\begin{table}[t]
\centering
\begin{threeparttable}
\caption{Noise schedules at an equal total sample budget. The rows marked (oracle) use the true-dynamics rollout, which carries no model error.}
\label{tab:anneal}
\footnotesize
\begin{tabularx}{\linewidth}{l Y Y}
\toprule
Schedule & Reach $1$\,cm & Median final error [m] \\
\midrule
$1 \times NS$ wide (MPPI) & $33/50$ & $0.010$ \\
$1 \times NS$ narrow (MPPI) & $15/50$ & $0.041$ \\
$S \times N$ wide & $46/50$ & $0.012$ \\
$S \times N$ narrow & $26/50$ & $0.045$ \\
$S \times N$ anneal (ours) & $\mathbf{50/50}$ & $\mathbf{0.0097}$ \\
\midrule
$S \times N$ narrow (oracle) & $10/10$ & $0.009$ \\
$S \times N$ anneal (oracle) & $10/10$ & $0.007$ \\
\bottomrule
\end{tabularx}
\end{threeparttable}
\end{table}

We then measure whether the rollout model meets the control period.
Table~\ref{tab:latency} reports the end-to-end planning time per control step, which covers every line of Algorithm~\ref{alg:bkmbd} over all five annealing stages.
The \bkmbd{} planner returned at a median of $11.0$\,ms with a worst case of $14.7$\,ms and missed no deadline over all $50$ trials, while the oracle took a median of $2598$\,ms and overran the $50$\,ms period at every one of its $1200$ control steps.
The gap follows from what one control step asks of the rollout model.
The $2$\,ms integrator turns $NST = 60{,}000$ rollout steps into $1.5$ million physics substeps, and the $ST = 75$ batched steps among them lie on a serial chain.
On the manipulator, a bilinear step issues $m + 1 = 8$ matrix--vector products instead of one, yet shared lifted-state batching raises the planner time only from $7.6$\,ms to $11.0$\,ms, less than a factor of $1.5$.

\subsection{The Contribution of the Annealed Schedule}\label{subsec:exp_anneal}

We settle whether the advantage of the annealed schedule comes from learned-rollout error by varying the noise schedule alone over the bilinear model.
The five schedules of Table~\ref{tab:anneal} share one total budget of $NS$ candidates per control step, written as stages $\times$ candidates per stage.
$1 \times NS$ draws the whole budget at once and $S \times N$ spreads it over $S$ stages, at the wide level $\sigma = 1.2$, the narrow level $\sigma = 0.3$, or annealed from wide to narrow.
By the third property of Sec.~\ref{subsec:mbd}, the single-stage conditions are the MPPI update at the same budget.
The obstacle-free task has no true-cost multimodality, so annealing gains cannot arise from it.

\begin{table}[t]
\centering
\begin{threeparttable}
\caption{Step size and step quality at a fixed noise level. Efficiency is the error removed per unit displacement.}
\label{tab:step}
\footnotesize
\begin{tabularx}{\linewidth}{l Y Y Y}
\toprule
$\sigma$ & $|\Delta U|$ & Error removed & Efficiency($\times 10^{3}$) \\
\midrule
$0.3$ & $0.109$ & $0.00665$ & $61.7$ \\
$0.8$ & $0.292$ & $0.00571$ & $19.2$ \\
$1.2$ & $0.418$ & $0.00409$ & $9.9$ \\
anneal & $0.109$ & $0.00525$ & $48.4$ \\
\bottomrule
\end{tabularx}
\end{threeparttable}
\end{table}

\begin{figure}[!tbp] 
\centering
\captionsetup{skip=2pt}
\begin{subfigure}{0.42\linewidth}
\includegraphics[width=\linewidth]{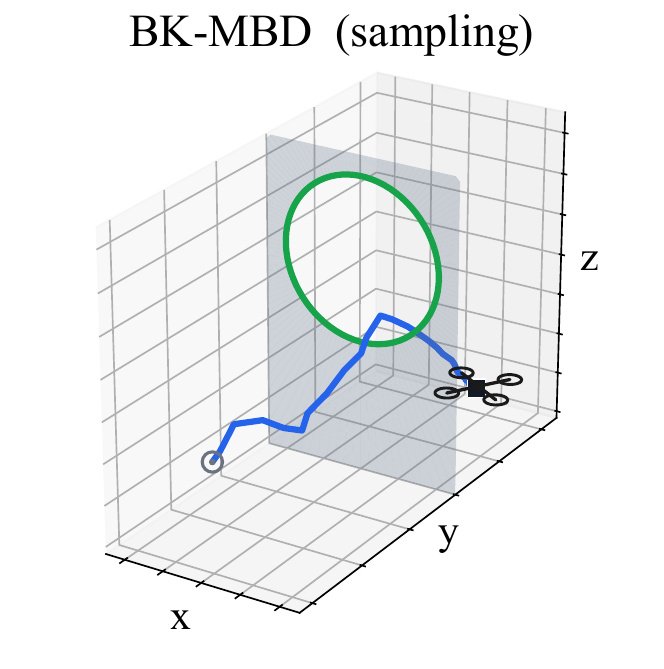}
\caption{}
\label{fig:window_bk}
\end{subfigure}
\hfill
\begin{subfigure}{0.42\linewidth}
\includegraphics[width=\linewidth]{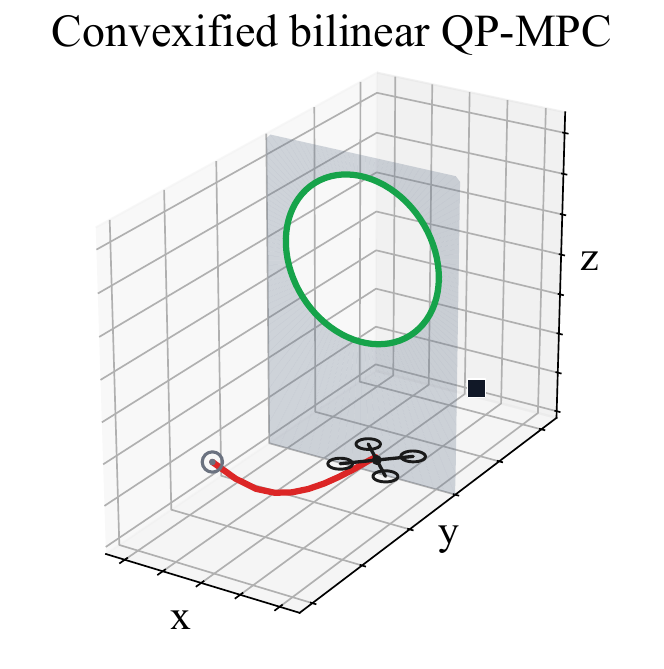}
\caption{}
\label{fig:window_qp}
\end{subfigure}
\caption{Non-convex passage on the drone. The wall admits a single circular aperture. (a) \bkmbd{} threads the aperture and reaches the far-side goal. (b) the convexified controller holds the near side.}
\label{fig:window}
\end{figure}

Table~\ref{tab:anneal} reports the outcome at the $1$\,cm tolerance, where the schedules separate.
The annealed schedule reached on all $50$ trials and left the lowest tracking error at the end of the budget.
It gains $17$ and $35$ goals over the single-stage wide and narrow MPPI conditions, and $4$ and $24$ over the conditions that repeat the same level $S$ times.
Neither the sample budget nor the application of several serial stages therefore accounts for the result, and the changing noise scale is the separating factor.

Table~\ref{tab:step} reports the displacement one control step produces and the tracking error it removes, averaged over control steps and trials, and their ratio as an efficiency.
From $\sigma = 0.3$ to $\sigma = 1.2$ the displacement grows by a factor of $3.8$ while the efficiency drops by a factor of six.
The annealed schedule instead executes at the displacement of the narrow level, $0.109$, and retains an efficiency of $48.4$ against the $9.9$ of the wide level.
The final stage produces the executed step, while the preceding wider stages have already exposed the iterate to a broader neighborhood of the current plan.

\begin{figure*}[t]
\centering
\captionsetup{skip=2pt}
\begin{subfigure}{0.49\linewidth}
\includegraphics[width=\linewidth]{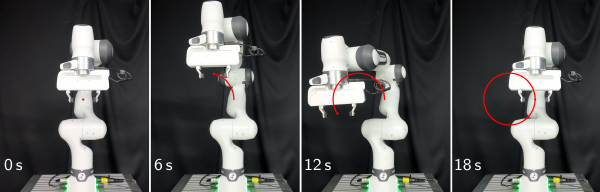}
\caption{}
\label{fig:hw_strip_bk}
\end{subfigure}
\hfill
\begin{subfigure}{0.49\linewidth}
\includegraphics[width=\linewidth]{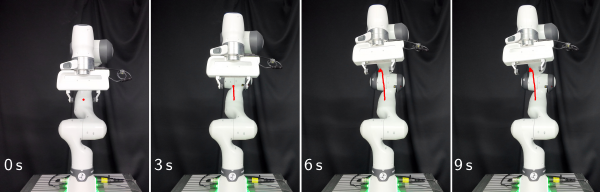}
\caption{}
\label{fig:hw_strip_dk}
\end{subfigure}

\vspace{0.2em}

\begin{subfigure}{0.49\linewidth}
\includegraphics[width=\linewidth]{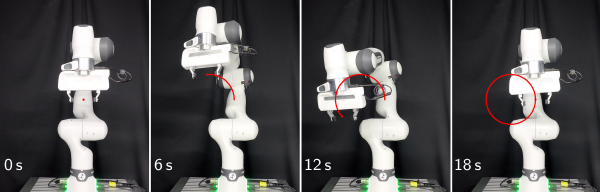}
\caption{}
\label{fig:hw_strip_split}
\end{subfigure}
\hfill
\begin{subfigure}{0.49\linewidth}
\includegraphics[width=\linewidth]{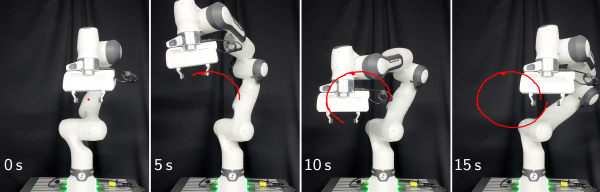}
\caption{}
\label{fig:hw_strip_mppi}
\end{subfigure}
\caption{Circular tracking on the physical FR3, with postures sampled at equal intervals over one revolution and the executed tool path traced in red. (a) \bkmbd{}, which completes the revolution. (b) \dkmbd{}, from the start of the command to the guard stop. (c) \dksplit{}. (d) \textsc{BK-MPPI}.}
\label{fig:hw_strip}
\end{figure*}

The model and the true dynamics disagree on the settled plan of the narrow schedule, $0.024$\,m against $0.053$\,m, and agree more closely on that of the annealed schedule, $0.020$\,m against $0.015$\,m.
Repeating the two schedules on the true-dynamics rollout, over the same goals and the same planner stream, removes the separation.
The narrow schedule reaches $10/10$ as the annealed schedule does, against $26/50$ under the learned rollout, and its settled error moves from $0.045$ to $0.009$\,m.
The narrow level settles quickly on either rollout, so what it loses under the learned one is consistent with sensitivity to the surrogate landscape rather than with an inability of the sampler to refine a plan.

\subsection{Non-Convex Passage}\label{subsec:exp_window}

We settle whether the structure of the free space is what separates the two optimizers by exchanging the optimizer alone.
A single circular aperture in a wall is the only route to the goals on the far side, and the straight chord to any goal crosses the wall below the opening.
The drone flies $50$ trials, and both planners receive the same trained bilinear model.
\bkmbd{} takes the wall as a cost penalty, while the convexified controller relinearizes the keep-out set into one convex region per step and solves the resulting QP over four iterations along the nominal rollout~\cite{bruder2021bilinear}.

\bkmbd{} reached the $1$\,cm tolerance on $50/50$ goals and the convexified controller on $10/50$, and \bkmbd{} executed no constraint violation although it treats the wall as a penalty alone.
Fig.~\ref{fig:window} draws the paths of the two planners.

The free space is the two sides of the wall joined through the aperture.
Any convex region that excludes the wall lies entirely on the near side, so the far-side goal is infeasible at every step and the solve converges to the point of that region closest to the goal.
The convexified controller settled a median of $0.030$\,m from the wall, the margin of its own keep-out set, and the ten goals it reached were those closest to the aperture axis, at most $0.148$\,m against a median of $0.291$\,m. A convexification of the dynamics and one of the free space act on it at once, but frozen dynamics would instead fail by tracking a poor nominal, so both point to the free space.

\begin{table}[t]
\centering
\begin{threeparttable}
\caption{Circular tracking on the physical FR3 over ten planner seeds.}
\label{tab:hw}
\footnotesize
\begin{tabularx}{\linewidth}{l Y Y Y}
\toprule
Method & Full run & RMSE [mm] & Median / worst [ms] \\
\midrule
\bkmbd{} (ours) & $10/10$ & $7.09 \pm 0.12$ & $22.9$ / $44.7$ \\
\dkmbd{} & $0/10$ & $39.75 \pm 0.67$ & $10.9$ / $23.4$ \\
\dksplit{} & $10/10$ & $4.97 \pm 0.09$ & $28.6$ / $83.0$ \\
\textsc{BK-MPPI} & $8/10$ & $15.76 \pm 5.23$ & $17.0$ / $53.1$ \\
\mbdtrue{} & $0/10$ & $41.22 \pm 0.07$ & $1617$ / $1762$ \\
\bottomrule
\end{tabularx}
\end{threeparttable}
\end{table}

\subsection{Validation on the Physical Robot}\label{subsec:exp_hw}

Five methods listed in Table~\ref{tab:hw}, track a target moving along a circle on the physical FR3 under the same CPU budget.
The target traverses a circle of radius $0.10$\,m and period $16$\,s in the $yz$ plane, and the trajectory is not given in advance, so the controller sees only the target position at each control step.
Fig.~\ref{fig:hw_strip} shows the postures over one revolution.
Each trial runs for $35$\,s, a $3$\,s lead-in followed by two revolutions.
A guard commands zero velocity and ends the trial when the Cartesian tracking error exceeds $80$\,mm for $0.25$\,s.
Each rollout class is run ten times, over planner seeds $0$ to $9$.
The models are identified from data recorded on this robot under the protocol of Sec.~\ref{subsec:setup}. \textsc{BK-MPPI} uses the single-stage narrow condition, which shows a distinct learned-rollout effect in Table~\ref{tab:anneal}.

Only \bkmbd{} and \dksplit{} completed all ten trials.
\dkmbd{} triggered a guard stop in every trial, while \textsc{BK-MPPI} completed eight and exceeded the guard in the remaining two.
\mbdtrue{} issued no nonzero command because its planning latency exceeded the control period.
\bkmbd{} returned at a median of $22.9$\,ms with a worst case of $44.7$\,ms and stayed within the $50$\,ms period over all ten trials, whereas \dksplit{} overran it at $83.0$\,ms, so \bkmbd{} is the only method here that met both the tracking task and the deadline.

The failure of \dkmbd{} lies in the representation and that of \mbdtrue{} in the computation.
The state-independent input derivative $C B_0$ of \dkmbd{} represents only the mean input effect over the training distribution, as reflected by the axis errors of Fig.~\ref{fig:hw_error} and its directional cosine of $0.318$, against $0.914$ for \bkmbd{}.
The successful tracking of \dksplit{} corroborates on hardware the representation limit identified in the preceding simulation experiment.
Likewise, the contrasting closed-loop outcomes of \bkmbd{} and \textsc{BK-MPPI} under the same BK rollout show that the benefit of the annealed schedule persists on the physical robot.
Rolling the settled plans of the completed \textsc{BK-MPPI} trials through the model and the robot gave terminal errors of $8.90 \pm 2.08$\,mm and $19.25 \pm 8.21$\,mm, respectively, with the robot error larger in all but one trial.
Because this discrepancy belongs to the learned BK rollout shared by both planners rather than to MPPI itself, the result supports the earlier conclusion that the annealed schedule is less sensitive to surrogate error than a single-stage narrow update.

\begin{figure}[t]
\centering
\includegraphics[width=\linewidth]{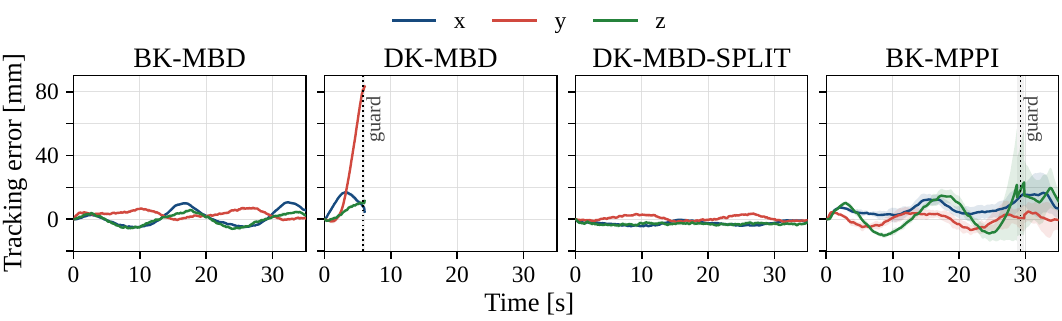}
\caption{Signed TCP tracking error per axis on the physical FR3, mean over ten planner seeds with one standard deviation. The dotted line marks the guard stop.}
\label{fig:hw_error}
\end{figure}

\section{Conclusion}\label{sec:conclusion}

We introduced \bkmbd{}, which runs the annealed sampling of MBD at the control rate by performing every rollout with a bilinear Koopman model identified from interaction data.
In simulation, exchanging one component of the controller at a time separated three effects: the input gain of the rollout, the noise schedule under a learned rollout, and the structure of the free space.
On a physical manipulator, \bkmbd{} tracked a moving target and met the control deadline at every step; no other method did both.
The input gain and the noise schedule kept the roles that simulation assigned them.
MBD owes its behavior in contact-rich settings to a rollout that contains the contact, and a lifted model would instead have to represent it.
Every task here is contact-free, so whether the substitution preserves that behavior is left for future work.

\bibliographystyle{IEEEtran}
\bibliography{IEEEabrv,references}

\end{document}